\def\year{2021}\relax
\documentclass[letterpaper]{article} 
\usepackage{aaai21}  
\usepackage{times}  
\usepackage{helvet} 
\usepackage{courier}  
\usepackage[hyphens]{url}  
\usepackage{graphicx} 
\usepackage{natbib}  
\usepackage{caption} 
\title{GLIB: Efficient Exploration for Relational Model-Based Reinforcement Learning via Goal-Literal Babbling}
\author{Rohan Chitnis$^*$, Tom Silver$^*$, Joshua Tenenbaum, Leslie Pack Kaelbling, Tom\'{a}s Lozano-P\'{e}rez\\\textnormal{MIT Computer Science and Artificial Intelligence Laboratory}\\\texttt{\textnormal{\{ronuchit, tslvr, jbt, lpk, tlp\}@mit.edu}}}

\usepackage{multicol}
\usepackage{color}
\usepackage{amssymb}
\usepackage{amsmath}
\usepackage{amsthm}
\usepackage{tikz}
\usepackage{chngcntr}
\usepackage[makeroom]{cancel}
\usetikzlibrary{arrows,decorations.markings}

\usepackage{booktabs}
\usepackage{graphicx}
\usepackage{mathtools}
\usepackage{bbm}
\usepackage[switch]{lineno}

\usepackage{caption}
\usepackage{subcaption}

\usepackage{algorithm2e}
\let\oldnl\nl
\newcommand{\nonl}{\renewcommand{\nl}{\let\nl\oldnl}}

\usepackage{indentfirst}

\newenvironment{tightlist}%
{\begin{list}{$\bullet$}{%
    \setlength{\topsep}{0in}
    \setlength{\partopsep}{0in}
    \setlength{\itemsep}{0in}
    \setlength{\parsep}{0in}
    \setlength{\leftmargin}{1.5em}
    \setlength{\rightmargin}{0in}
}
}%
{\end{list}
}

\newcommand{\secref}[1]{(Section \ref{#1})}

\newcommand{\figref}[1]{Figure~\ref{#1}}
\newcommand{\algref}[1]{Algorithm~\ref{#1}}

\newcommand{\tabref}[1]{Table~\ref{#1}}
\newcommand{\thmref}[1]{Theorem~\ref{#1}}

\newtheorem{theorem}{Theorem}
\newtheorem{lem}{Lemma}
\newtheorem{defn}{Definition}

\newtheorem{cor}{Corollary}

\counterwithin*{thm}{section}

\def\thickhline{%
  \noalign{\ifnum0=`}\fi\hrule \@height \thickarrayrulewidth \futurelet
   \reserved@a\@xthickhline}
\def\@xthickhline{\ifx\reserved@a\thickhline
               \vskip\doublerulesep
               \vskip-\thickarrayrulewidth
             \fi
      \ifnum0=`{\fi}}

\newlength{\thickarrayrulewidth}
\newenvironment{problemsetting}[1][htb]
  {
   \begin{algorithm}[#1]%
  }{\end{algorithm}}
\newcommand{\irale}{\textsc{iral}{\small e}}
\newcommand{\expo}{\textsc{expo}}
\newcommand{\rex}{\textsc{rex}}
\newcommand{\ilm}{\textsc{ilm}}
\newcommand{\glib}[1]{\textsc{glib{\small {#1}}}}
\newcommand{\glibg}[1]{\textsc{glib-g{\small {#1}}}}
\newcommand{\glibl}[1]{\textsc{glib-l{\small {#1}}}}

\newcommand{\zpk}{\textsc{lndr}}

\newcommand{\I}{\mathcal{I}}

\renewcommand{\P}{\mathcal{P}}

\newcommand{\Q}{\mathcal{Q}}
\renewcommand{\S}{\mathcal{S}}
\newcommand{\A}{\mathcal{A}}

\newcommand{\D}{\mathcal{D}}

\DeclarePairedDelimiterX{\infdivx}[2]{(}{)}{%
  #1\;\delimsize\|\;#2%
}

\begin{document}
\maketitle

\begin{abstract}
We address the problem of efficient exploration for transition model learning in the relational model-based reinforcement learning setting without extrinsic goals or rewards.
Inspired by human curiosity, we propose goal-literal babbling (\glib{}), a simple and general method for exploration in such problems.
\glib{} samples relational conjunctive goals that can be understood as specific, targeted effects that the agent would like to achieve in the world, and plans to achieve these goals using the transition model being learned.
We provide theoretical guarantees showing that exploration with \glib{} will converge almost surely to the ground truth model.
Experimentally, we find \glib{} to strongly outperform existing methods in both prediction and planning on a range of tasks, encompassing standard PDDL and PPDDL planning benchmarks and a robotic manipulation task implemented in the PyBullet physics simulator. Video: \url{https://youtu.be/F6lmrPT6TOY} Code: \url{https://git.io/JIsTB}

\end{abstract}

\section{Introduction}
\label{sec:intro}

Human curiosity often manifests in the form of a question: ``I wonder if I can do X?''
A toddler wonders whether she can climb on the kitchen counter to reach a cookie jar.  
Her dad wonders whether he can make dinner when he's missing one of the key ingredients.
These questions lead to actions, actions may lead to surprising effects, and from this surprise, we learn.
In this work, inspired by this style of playful experimentation \cite{gil1994learning,cropper2019playgol}, we study exploration via goal-setting for the problem of learning \emph{relational transition models} to enable robust, generalizable planning.

\begin{figure}[t]
  \centering
    \noindent
    \includegraphics[width=\columnwidth]{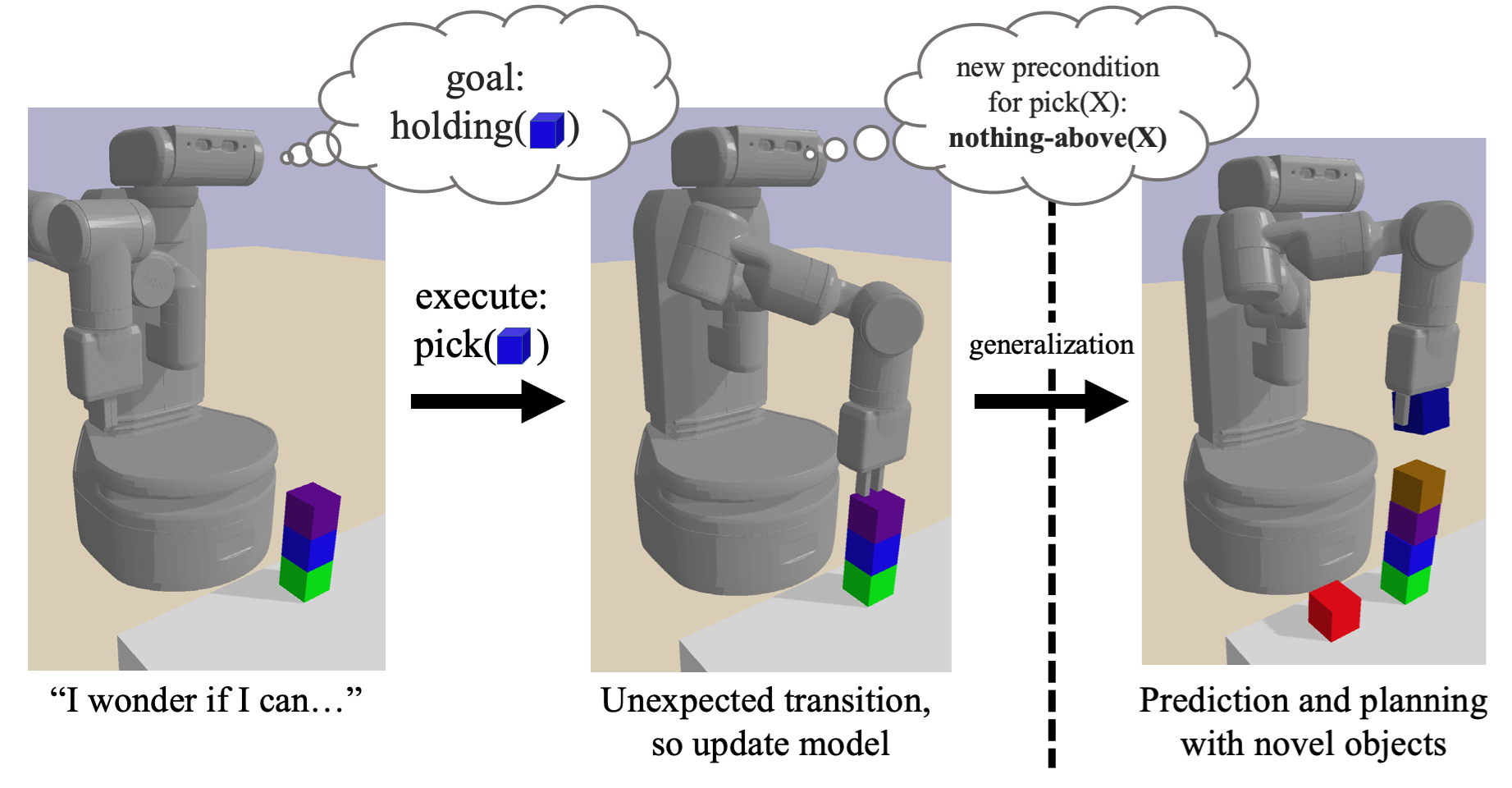}
    \caption{We study goal babbling as a paradigm for exploration in relational transition model learning. \emph{Left:} A robot in our PyBullet domain sets itself a goal of \texttt{holding} the blue object. \emph{Middle:} The robot (mistakenly) believes the goal can be achieved by executing \texttt{pick} on the blue object. When it tries this plan, it fails due to the purple object in the way. From this and previous data, the robot can update its transition model. In this case, the robot learns that a precondition of picking an object is that nothing is on top of it. \emph{Right:} In more complex environments involving novel objects, the robot can plan with its learned model to achieve goals.}
  \label{fig:teaser}
\end{figure}

Transition model learning is central to model-based reinforcement learning (RL), where an agent learns an approximate transition model through online interaction with its environment.
The learned transition model can be used in combination with a planner to maximize a reward function or reach a goal.
If the transition model is relational, that is, represented in terms of lifted relations between objects in the environment, then it generalizes immediately to problems involving new and different objects than those previously encountered by the agent~\cite{dvzeroski2001relational,tadepalli2004relational}.

In this paper, we address the problem of \emph{efficient exploration} for online relational transition model learning.
This setting isolates the exploration problem in model-based RL, and can be understood as model-based RL without an extrinsic reward function.
Previous approaches to exploration for relational model-based RL have considered extensions of classical tabular methods like \textsc{R-max} and E$^3$ to the relational regime ~\cite{lang2012exploration,ng2019incremental}.
Prior work in the AI planning literature has also considered exploration for learning and refining planning operators~\cite{gil1994learning,Shen94autonomouslearning,wang1996planning,rodrigues2011active}.
In practice, these approaches tend to be myopic, exploring locally and cautiously while leaving far-away regions of the state space unexplored.

In pursuit of an exploration strategy that can drive an agent toward interesting regions of the state space, we propose a novel family of exploration methods for relational transition model learning called \emph{Goal-Literal Babbling} (\glib{}).
The basic approach is illustrated in Figure \ref{fig:teaser}.
Goals in \glib{} are conjunctions of literals (relations); informally, these can be understood as specific, targeted effects that the agent would like to achieve in the world.
Goals are proposed according to a novelty measure~\cite{lehman2008exploiting}.
A particular instantiation of \glib{} is characterized by an integer $k$, which bounds the number of literals involved in each goal conjunction, and a choice between lifted or ground goals.
Lifted \glib{} (\glibl{}) drives the agent to situations that are radically new, like creating a stack of three blocks for the first time.
Ground \glib{} (\glibg{}) may be preferable when interesting transitions are difficult to express with a short lifted conjunction.

To try to achieve the goals babbled by \glib{}, we plan using the current (often flawed) transition model that we are in the process of learning.
Conventional wisdom suggests that planning with incorrect models should be avoided due to the potential for compounding errors, especially in a learning-to-plan setting, where these errors could lead the agent to build a model that is incorrect and irreparable.
However, we show both in theory and in practice that this intuition does not apply: we provide theoretical guarantees that \glib{} cannot get stuck in a subregion of the reachable state space \secref{subsec:theory}, and we show empirically that \glib{} yields very strong performance across a range of tasks \secref{sec:experiments}.

This work has the following contributions. 
(1) We propose \glib{}, a novel family of exploration methods for online relational transition model learning.
(2) We prove that exploration under \glib{} is almost surely recurrent (in the sense of Markov chains) and almost surely converges to the ground truth model 
given mild assumptions on the planner, learner, and domain.
(3) We evaluate model prediction error and planning performance across six tasks, finding \glib{} to consistently outperform several prior methods.
(4) We assess the extent to which \glib{} is sensitive to the particular choice of model-learning algorithm, finding \glib{} to be the best performing exploration method regardless of the model learner.
We conclude that \glib{} is a simple, strong, and generally applicable exploration strategy for relational model-based RL.
\section{Related Work}
\label{sec:related}

\textbf{Learning and Refining Planning Operators.}
Learning relational transition models has been the subject of a long line of work in the planning literature \cite{benson1995inductive,rodrigues2011active,cresswell2013acquiring,zhuo2010learning,arora2018review}.
These methods, which focus on the learning problem, rather than the exploration problem, are typically offline, assuming a fixed dataset as part of the problem specification.
Our focus is on \emph{online} exploration problem, where the agent must collect its own data.

Other work in the planning community has considered planning operator \emph{refinement}: gathering data to improve an imperfect set of operators \cite{gil1994learning,Shen94autonomouslearning,wang1996planning}.
These methods are useful when one has a decent model in hand and an error is discovered; they suggest actions for gathering data to correct the error.
Existing methods for relational model refinement, such as \textsc{expo}~\cite{gil1994learning},
can be combined with a fallback strategy that selects actions when no existing operators require refinement.
However, these methods lack \emph{intrinsic drive}; there is nothing directing the agent toward unexplored regions of the state space.
We use \expo{} as a baseline in our experiments.

\textbf{Exploration in Model-Based RL.}
Exploration is one of the fundamental challenges of reinforcement learning (RL).
Exploration strategies for model-based RL are particularly relevant to our  setting, though typically the agent is given rewards to optimize.
E$^3$~\cite{kearns2002near} and \textsc{R-max}~\cite{brafman2002r} are two such classic strategies.
Dyna~\cite{dyna} was one of the first RL approaches to plan with the model being learned.
More recently, \citet{sekar2020planning} consider exploration for deep model-based RL, using learned transition models to plan in a latent space via backpropagation; this work shares with ours the idea of planning for exploration.
In contrast to all these works, though, our focus is on the \emph{relational} regime.

\citet{walsh2010efficient} proves the existence of a \textsc{kwik} (knows what it knows) algorithm for efficient exploration in relational model learning.
As pointed out by \citet{lang2012exploration}, Walsh's algorithm provides theoretical insight but has never been realized in practice.
\citet{lang2012exploration} propose \rex{}, which extends E$^3$ to the relational regime.
\irale{}~\cite{rodrigues2011active} learns lifted rule-based transition modules using the heuristic that an action should be explored if its preconditions \emph{almost} hold.
Unlike other methods, \irale{} does not perform lookahead with the learned model.
Recent work by \citet{ng2019incremental} proposes \ilm{}, an extension of \rex{} that incorporates a notion of model \emph{reliability} into the count function that is used to determine whether a state is worth exploring.
We include \rex{}, \irale{}, and \ilm{} as baselines in our experiments.

\textbf{Goal Babbling, Robotics, Deep RL.}
Our use of the term ``babbling'' is an homage to prior work in robotics on goal babbling, originally proposed for learning kinematic models \cite{goalbabbling1,baranes2013active}.
\citet{forestier2016modular} use goal babbling in a continuous model-based setting where trajectory optimization suffices for planning.
Other recent work considers goal babbling for automatic curriculum generation in model-free deep RL
\cite{florensa2017automatic,forestier2017intrinsically,laversanne2018curiosity,campero2020learning}.
For instance, \citet{nair2018visual} consider ``imagining'' goals for RL from visual inputs.
Our work continues this line of work on goal babbling for exploration, but unlike these prior works, we are interested in learning \emph{relational} models that are amenable to symbolic planning.
\section{Problem Setting}
\label{sec:prelim}

We study exploration for online transition model learning in stochastic, relational domains. 
As in typical RL settings, an agent interacts episodically and online with a (fully observable) \emph{environment}, defined by a state space $\S$; action space $\A$; transition model $P(s' \mid s, a)$ with $s, s' \in \S, a \in \A$; initial state distribution $\I$; and episode length $T$. The agent does \emph{not} know the transition model, but it does know $\S$ and $\A$. As it takes actions in the environment, the agent observes states sampled from the transition model.

A \emph{predicate} is a Boolean-valued function. A predicate applied to objects (resp. variables) is a ground (resp. unground or lifted) \emph{literal}.
Objects and variables may be typed or untyped.
All states $s \in \S$ are \emph{relational} with respect to a known set of predicates $\P$; that is, each $s$ is represented as a set of ground literals constructed from the predicates in $\P$. Any ground literal not in $s$ is considered to be false.
The set of objects is finite and fixed within an episode but varies between episodes.
Actions in our setting are also relational over the same object set; an action $a \in \A$ is a ground literal constructed from a known set of predicates $\Q$.
Since $\P$, $\Q$ and the set of objects are all finite, the state and action spaces are also finite (but typically very large).

\textbf{Evaluation.} Unlike in typical RL settings, here the agent does not have a reward function to optimize;  rather, its objective is to learn a model that is as close as possible to the true environment transition model~\cite{ng2019incremental}. We measure the quality of the learned model by evaluating its prediction error on random (state, action) pairs. However, we are also interested in the agent's ability to \emph{use} its learned model to solve tasks via planning.
We therefore also measure the quality of the learned model by testing it on a set of \emph{planning problems}, where each planning problem is made up of an initial state and a goal (a binary classifier, expressed in predicate logic, over states).
For each planning problem, the agent uses a planner and its learned model to find a policy $\pi : \mathcal{S} \to \mathcal{A}$. This planner may return \emph{failure} if it is unable to find a policy. If a policy is returned, it is executed from the initial state for a fixed horizon or until the goal is reached.\footnote{In deterministic environments, a planner returns a sequential plan rather than a policy. Going forward, we will not make this distinction between plans and
policies; at an intuitive level, a planner simply produces actions that drive the agent toward a given goal.}

\textbf{The Importance of Exploration.}
The overall problem setting is summarized in the pseudocode above.
As the agent interacts with the environment, it builds a dataset $\D$ of transitions it has seen thus far, and uses these transitions to learn a model $\hat{P}(s' \mid s, a)$.
The accuracy of this model will depend critically on the quality of the dataset $\D$.
The \textsc{Explore} method is responsible for choosing actions to gather the dataset; it must guide the agent through maximally informative parts of the transition space.
Our objective in this paper is to design an \textsc{Explore} method that efficiently gathers data and leads to good prediction and planning performance with as few environment interactions as possible.

\begin{problemsetting}[t]
  \SetAlgoNoEnd
  \DontPrintSemicolon
  \SetKwFunction{algo}{algo}\SetKwFunction{proc}{proc}
  \SetKwProg{myalg}{Problem Setting}{}{}
  \SetKwProg{myproc}{Subroutine}{}{}
  \SetKw{Continue}{continue}
  \SetKw{Break}{break}
  \SetKw{Return}{return}
  \myalg{\textsc{Online Model Learning}}{
    \nonl \textbf{Input:} State space $\S$ and action space $\A$.\;
    \nonl \textbf{Input:} Sampler from initial state distribution $\I$.\;
    \nonl \textbf{Input:} Episode length $T$.\;
    \nonl \textbf{Initialize:} $\D \gets \emptyset$, a dataset of transitions.\;
    \nonl \textbf{Initialize:} $\hat{P}$, an initial transition model estimate.\;
    \nonl \While{$\hat{P}$ is still improving}{
    \nonl $s \sim \I$\;
    \nonl \For{$T$ timesteps}{
    \nonl $a \gets $\textsc{Explore}$(\S, \A, \D, \hat{P}, s)$\;
    \nonl Execute $a$, observe next state $s'$.\;
    \nonl $\mathcal{D} \gets \mathcal{D} \cup \{ (s, a, s') \}$\;
    \nonl $\hat{P} \gets $ \textsc{LearnModel}$(\D, \hat{P})$\;
    \nonl $s \gets s'$\;
    }}
    \nonl \Return final learned model $\hat{P}$\;
    }
\end{problemsetting}

\section{Relational Learning and Planning}
\label{sec:modellearning}

Online transition model learning requires implementations of \textsc{LearnModel} and \textsc{Exploration}; our focus in this work is the latter, which we address in Section \ref{sec:approach}.
In this section, we briefly describe the two existing techniques for \textsc{LearnModel} that we use in our experiments.

Following prior work on exploration for relational model-based RL \cite{lang2012exploration,ng2019incremental}, we consider transition models that are parameterized by \emph{noisy deictic rules} \cite{pasula2007learning}. 
A noisy deictic rule (NDR) is made up of an unground action literal, a set of \emph{preconditions}, which are (possibly negated) unground literals that must hold for the rule to apply, and a categorical distribution over \emph{effects}, where each possible outcome is a set of (possibly negated) unground literals whose variables appear in the action literal or preconditions.
An NDR effect distribution may include a special \emph{noise outcome} to capture any set of effects not explicitly modeled by the other elements of the distribution.
An NDR \emph{covers} a state $s$ and action $a$ when there exists a binding of the NDR's variables to objects in $(s, a)$ that satisfy the action literal and preconditions of the NDR.
Each action predicate is associated with a \emph{default} NDR, which covers $(s, a)$ when no other NDR does.
A collection of NDRs is a valid representation of a transition model when exactly one NDR covers each possible $(s, a)$.
The associated distribution $P(s' \mid s, a)$ is computed by 1) identifying the NDR that covers $(s, a)$, 2) grounding the effect sets with the associated binding, and 3) applying the effects (adding positive literals and removing negative literals) to $s$ to compute $s'$.

\citet{pasula2007learning} propose a greedy search algorithm for learning a collection of NDRs;
we call this method Learning NDRs (\zpk{}) and use it as our first implementation of \texttt{LearnModel}, following prior work~\cite{lang2012exploration,ng2019incremental}.\footnote{The exploration method used in this original \zpk{} work is called ``action babbling'' in our experiments.}
To assess the extent to which the relative performance of different implementations of \texttt{Explore} are dependent on the transition model learner, we also consider a second implementation of \texttt{LearnModel}, \textsc{tilde}~\cite{foldt}, which is an inductive logic programming method for rule learning in deterministic domains.

Noisy deictic rules are plug-compatible with PPDDL \cite{younes2004ppddl1}, the probabilistic extension of PDDL, which is a standard description language for symbolic planning problems.
PPDDL planners consume a specification of the transition model, initial state, and goal, and return a policy.
FF-Replan \cite{ffreplan} is a PPDDL planner that determinizes the transition model and calls the FastForward planner \cite{ff}, replanning when an observed transition does not match the determinized model.
We use FF-Replan with single-outcome determinization as our planner for all experiments.

\section{Exploration via Goal-Literal Babbling}
\label{sec:approach}

In this section, we describe \glib{} (goal-literal babbling), our novel implementation of the \textsc{Explore} method for online relational transition model learning. 
See \algref{alg:glib} for pseudocode and \figref{fig:timeline} for an illustration of \glib{} in the Keys and Doors domain \cite{konidaris2007building}.

\begin{figure*}[t]
  \centering
    \noindent
    \includegraphics[width=\textwidth]{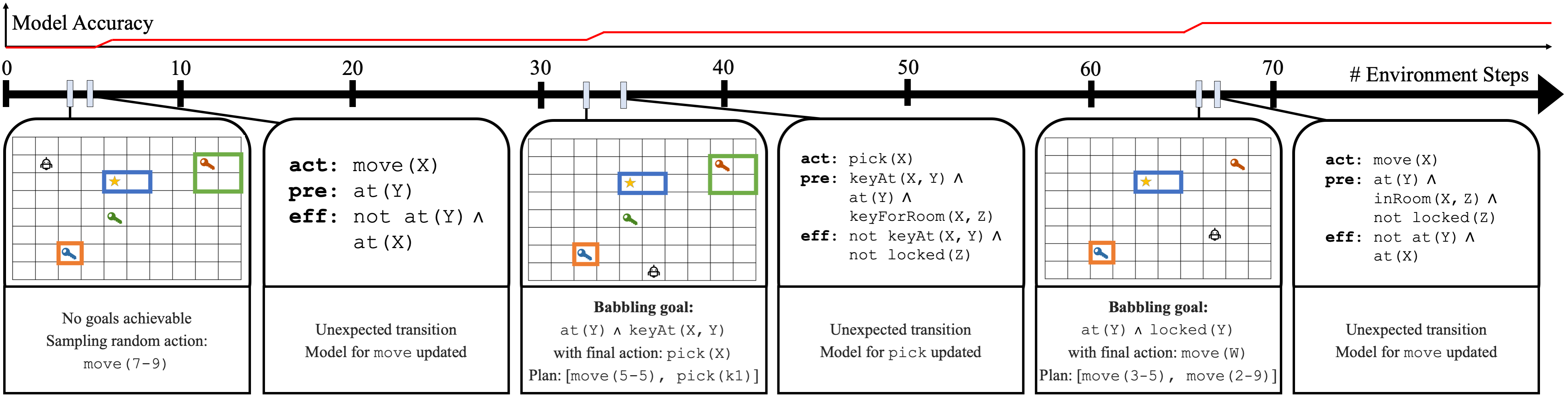}
    \caption{An agent exploring with \glibl{} in Keys and Doors. The agent begins with a trivial transition model that predicts empty effects for all actions. Under this model, no goals are achievable, so the agent randomly samples \texttt{move(7-9)}. Observing the subsequent transition, the agent updates its transition model for \texttt{move}, but overgeneralizes, believing now that moving anywhere is possible, when in fact the agent may only move to unlocked rooms. (NDR probabilities are omitted for clarity.)
    Later, the agent babbles a goal and action that induce a two-step plan to move to and pick the green key. After executing the plan, the agent updates its model for \texttt{pick}. Finally, the agent babbles another goal and action that induce a plan to move to a locked location. Observing the failure of this plan, the agent updates its model for \texttt{move}, correcting its previous overgeneralization. }
  \label{fig:timeline}
\end{figure*}

\subsection{Goal-Literal Babbling (\glib{})}
\label{subsec:glib}
\glib{} builds on the intuition that an exploration method should drive an agent to large, interesting regions of the transition space, even when such regions are far from the agent's initial state.
To this end, the first key idea of \glib{} is that the agent should randomly set itself, or \emph{babble}, goals that are conjunctions of a small number of literals. 
Intuitively, these goal literals represent a targeted set of effects that the agent would like to achieve in the world.
For example, in the Keys and Doors domain (\figref{fig:timeline}), the agent may rapidly move to a location with a key by setting itself the goal $\exists \texttt{X,Y. at(X)} \land \texttt{keyAt(Y,\;X)}.$
\glib{} has two main parameters: $k$, an upper bound on the conjunction size; and a mode, representing whether the chosen goals are \emph{lifted}, as in the example above, or \emph{ground}, as in \texttt{at(3-3)}.

The second important aspect of \glib{} is that each goal literal $G$ is proposed not in isolation, but together with an action $a_G$ that the agent should execute if and when that goal is achieved. 
The motivation for babbling actions in addition to goals is that to learn an accurate transition model, the agent must thoroughly explore the space of \emph{transitions} rather than states.
A proposed goal-action pair $(G, a_G)$ can be interpreted as a transition that the agent would like to observe.
Like the goals, actions can be ground or lifted, optionally sharing variables with the goal in the lifted case.
For example, in the Keys and Doors domain, the agent might babble the action $\texttt{pick(Y)}$ alongside the goal $\exists \texttt{X,Y.at(X)} \land \texttt{keyAt(Y,\;X)}$, indicating that it should pick the key while it is at the key's location (\figref{fig:timeline}).

If we were to naively sample goals uniformly from all possible ($\leq k$)-tuples of literals, this may lead the agent to repeatedly pursue the same goals. Instead, \glib{} uses a novelty measure \cite{lehman2008exploiting}, only selecting goals that have never appeared as a subset of any previous state.
For this reason, the \textsc{SampleGA} method in \algref{alg:glib} takes as input the current dataset $\D$. 
In practice, we use caching to make this computation very efficient. Empirically, we found that sampling only novel goals is imperative to the overall performance of \glib{}.

Once the agent has selected a goal-action pair $(G, a_G)$, it uses a planner to find a policy for achieving $G$ from the current state $s$, under the current learned model $\hat{P}$. If a policy is found (i.e., the planner does not return \emph{failure}), $a_G$ is appended as its final action. In \emph{lifted} mode, $a_G$ will be lifted, so we first run \textsc{GroundAction}, which grounds $a_G$ by randomly sampling values for any variables that are not already bound in the goal; see \figref{fig:timeline} for examples. We then execute the policy until it terminates.
If a policy is not found after $N$ tries, we fall back to taking a random action.

The choice of mode (\emph{ground} or \emph{lifted}) can have significant effects on the performance of \glib{}, and the best choice depends on the domain. On one hand, novelty in \emph{lifted} mode has the tendency to over-generalize: if \texttt{location5} is the only one containing an object, then lifted novelty cannot distinguish that object being at \texttt{location5} versus elsewhere. On the other hand, novelty in \emph{ground} mode may not generalize sufficiently, and so can be much slower to explore.

\textbf{Filtering out unreasonable goals.} We ensure that the goals babbled by \textsc{SampleGA} are reasonable by filtering out two types of goals: \emph{static} goals and \emph{mutex} goals. Static goals are goals where every literal in the conjunction is predicted to be left unchanged by the current learned model $\hat{P}$; such a goal will be either always true or always false under $\hat{P}$, and is therefore not useful to babble. Mutex goals are goals where any pair of literals in the conjunction cannot be satisfied simultaneously under $\hat{P}$; such a goal can never hold, and there is no use in expending planning effort to try to find a policy for achieving it. Mutex detection is known to be as hard as planning in the worst case, but there is a rich body of work on approximation methods~\cite{mutex1,mutex2,fd}. In this paper, we use MMM~\cite{mmmmutex}, a Monte Carlo algorithm that extracts approximate mutexes from a set of samples of reachable states, which we obtain through random rollouts of $\hat{P}$.
Note that the sets of static and mutex goals must be recomputed each time $\hat{P}$ is updated.

\subsection{Is Planning for Exploration Wise?}
\glib{} rests on the assumption that planning with a faulty transition model can ultimately lead to a better model.
In general, planning with faulty models is risky: prediction errors will inevitably compound over time.
However, when planning for exploration in particular, it is important to distinguish two failure cases: (1) a policy is found with the learned model and it does not execute as expected;
(2) no policy is found, even though one exists under the true model.
Interestingly, (1) is not problematic; in fact, it is ideal, because following this policy gives useful data to improve the model.
The only truly problematic case is (2).
\citet{wang1996planning} identifies a similar problem and attempts to reduce its occurrence by using a learning algorithm that errs on the side of ``more general'' models.
In our setting, a ``most general'' model is not well-defined.
Instead, if no policy is found after $N$ babbling tries, we fall back to a random action.
This allows us to escape situations where no goals seem possible.

\begin{algorithm}[t]
  \SetAlgoNoEnd
  \DontPrintSemicolon
  \SetKwFunction{algo}{algo}\SetKwFunction{proc}{proc}
  \SetKwProg{myalg}{Algorithm}{}{}
  \SetKwProg{myproc}{Subroutine}{}{}
  \SetKw{Continue}{continue}
  \SetKw{Break}{break}
  \SetKw{Return}{return}
  \myalg{\textsc{Explore: Goal-Literal Babbling}}{
    \nonl \textbf{Input:} $\S, \A, \D, \hat{P}, s$. \tcp*{\footnotesize See \secref{sec:prelim}.}
    \nonl \textbf{Hyperparameter:} Bound on literal count $k$.\;
    \nonl \textbf{Hyperparameter:} The mode [\emph{ground} or \emph{lifted}].\;
    \nonl \textbf{Hyperparameter:} Number of sampling tries $N$.\;
    \nonl \textbf{Internal state:} Policy in progress $\pi$. Starts \textbf{null}.\;
    \nonl \If{$\pi$ is not \emph{\textbf{null}}}
    {
    \nonl \Return $\pi(s)$\;
    }
    \nonl \For{$N$ iterations}
    {
    \tcp{\footnotesize Sample novel goal-action pair.}
    \nonl $(G, a_G) \gets$ \textsc{SampleGA}($\S, \A, \D, k, \text{mode}$)\;
    \tcp{\footnotesize Plan from current state.}
    \nonl $\pi \gets$ \textsc{Plan}($s$, $G$, $\hat{P}$)\;
    \nonl \If{$\pi$ found}
    {
    \nonl \If{mode is \emph{lifted}}{
    \nonl $a_G \gets$ \textsc{GroundAction}($a_G, G, s, \pi$)
    }
    \nonl Make $a_G$ be the final output of $\pi$.\;
    \nonl \Return $\pi(s)$\;
    }
    }
    \tcp{\footnotesize Fallback: random ground action.}
    \nonl \Return
    \textsc{Sample}($\A$)\;
    }
\caption{\small{Pseudocode for the goal-literal babbling (\glib{}) family of algorithms. See text for details.}
}
\label{alg:glib}
\end{algorithm}

\subsection{Theoretical Guarantees}
\label{subsec:theory}
We now present theoretical guarantees for the asymptotic behavior of \glib{}. 
Our main theorem gives conditions under which exploration with \glib{} is \emph{almost surely (a.s.) recurrent}; with probability 1, the agent will not get infinitely stuck in any subregion of the transition space.
We follow with a corollary that gives conditions under which the learned transition model will converge (a.s.) to the ground truth.

We say that a state $s \in \S$ is \emph{reachable} if there exists any sequence of at most $T$ actions that, with positive probability, leads to $s$ from an initial state.
Let $\Omega$ be the set of all reachable transitions: state-action pairs $(s, a)$ where $s$ is any reachable state in $\S$ and $a$ is any action in $\A$.
Note that any policy $\pi$ induces a Markov chain over state-action pairs $(s, a)$.
Let $\textsc{MC}(\pi, \I, \S, \A)$ denote this Markov chain, and let \textsc{Random} denote a uniformly random policy.
Let $S^t$ and $A^t$ be random variables for the state and action at time $t$.
\begin{defn}[Recurrent environment]
\label{defn:recurrent}
A \emph{recurrent} environment is one in which the Markov chain $\textsc{MC}(\textsc{Random}, \I, \S, \A)$ is recurrent over $\Omega$, that is, $\forall (s, a) \in \Omega, \forall t \ge 0, \exists t' > t$ s.t. $Pr(S^{t'} = s, A^{t'} = a) > 0$.
\end{defn}
Informally, a recurrent environment is one in which a random policy will infinitely revisit all reachable states.
\begin{defn}[$\epsilon$-sound planner]
\label{defn:soundplanner}
A planner $\textsc{Plan}$ is $\epsilon$-\emph{sound} if for any state $s$, goal $G$, and transition model $\hat{P}$, $\textsc{Plan}(s, G, \hat{P})$ returns a policy $\pi$ only if following $\pi$ from $s$ reaches $G$ within the horizon $T$ in model $\hat{P}$ with probability at least $\epsilon$. If no such $\pi$ exists, $\textsc{Plan}(s, G, \hat{P})$ reports failure.
\end{defn}
If an $\epsilon$-sound planner returns a policy, that policy is guaranteed to have at least $\epsilon$ probability of succeeding. (If the planner reports failure, there are no guarantees.)
\begin{defn}[Consistent learner]
\label{defn:consistent}
A transition model learner \textsc{LearnModel} is \emph{consistent} if for all $s \in \S, a \in \A$, the estimate $\hat{P}(S^{t+1} \mid S^t=s, A^t=a)$ converges a.s.~\cite{stout1974almost} to the ground truth $P(S^{t+1} \mid S^t=s, A^t=a)$ as samples are drawn from the latter.
\end{defn}
The following Lemma says, given a consistent learner, a goal, and a policy, we will a.s. either reach the goal, or learn a model under which the policy cannot reach the goal.
\begin{lem}
\label{lem:lemma1}
Suppose that \textsc{LearnModel} is consistent. Given any state $s \in \S$, goal $G$, and policy $\pi$, consider transitions sampled from the ground truth distribution $P$ by repeatedly starting at $s$ and following $\pi$ for $T$ steps.
Let $\hat{P}_t$ be the transition model returned by calling \textsc{LearnModel} on the first $t$ transitions.
Then a.s., either
(1) $G$ is eventually reached; or
(2) the probability that $\pi$ would reach $G$ from $s$ under $\hat{P}_t$ converges to 0 as $t \to \infty$.
\end{lem}
\begin{proof}
See Appendix \ref{app:proofs}.
\end{proof}
\begin{theorem}[\glib{} is a.s. recurrent]
\label{thm:recurrent}
Suppose that the environment is recurrent, \textsc{LearnModel} is consistent, and \textsc{Plan} is $\epsilon$-sound. Then for any integer $k > 0$, $\textsc{MC}(\glib{}(k), \I, \S, \A)$ is a.s. recurrent over $\Omega$.
\end{theorem}
\begin{proof}
See Appendix \ref{app:proofs}.
\end{proof}

\begin{defn}[Sufficiently representative]
\label{defn:suffrep}
Given a consistent learner \textsc{LearnModel}, a set of state-action pairs $\Gamma \subseteq \S \times \A$ is \emph{sufficiently representative} if the learned transition model $\hat{P}$ converges a.s. to the ground truth model $P$ as transitions starting from $(s, a) \in \Gamma$ are drawn from $P$.
\end{defn}
\begin{cor}
Suppose $\Omega$ is sufficiently representative.
Then under the assumptions of Theorem \ref{thm:recurrent},
the model learned from following \glib{} will converge a.s. to the ground truth model.
\end{cor}
\begin{proof}[Proof sketch]
By \thmref{thm:recurrent}, the Markov chain induced by \glib{} is a.s. recurrent over $\Omega$; thus, all state-action pairs $(s, a) \in \Omega$ are revisited infinitely many times. By Definition \ref{defn:suffrep}, $\hat{P}$ will a.s. converge to $P$.
\end{proof}

The consistency and $\epsilon$-soundness assumptions are mild and hold, respectively, for the implementations of \textsc{LearnModel} and \textsc{Plan} that we use in experiments.
The assumption of environment recurrence also holds for the environments we consider in our problem setting, because the interaction is episodic; every $T$ timesteps, a new initial state is sampled, guaranteeing that all reachable states will get visited infinitely often under a uniformly random policy.

The challenge of proposing a practical exploration method with strong sample-complexity guarantees still remains open.
\citet{walsh2010efficient} and \citet{mehta2011efficient} provide algorithms with guarantees that are intractable in practice; \citet{rodrigues2011active} and \citet{lang2012exploration} provide practical algorithms without guarantees.
To compare \glib{} against previous practical methods, we now turn to empirical investigations.
\section{Experiments}
\label{sec:experiments}

In this section, we present empirical results for \glib{} and several baselines.
We begin by describing the experimental setup, with additional details in Appendix \ref{app:details}.

\subsection{Experimental Setup}

\textbf{Domains.} We evaluate on six domains: three classical PDDL planning tasks, two benchmark PPDDL planning tasks, and one simulated robotic manipulation task.
\begin{tightlist}
\item \textbf{Blocks}~\cite{ipc}. This is the classic IPC deterministic Blocksworld domain, containing an agent that can pick, place, stack, and unstack blocks on a table. We train and evaluate on problems containing between 5 and 7 objects, yielding between 26 and 50 state literals.
\item \textbf{Gripper}~\cite{ipc}. This is the classic IPC deterministic Gripper domain, containing an agent that can move, pick, and drop a ball. We train and evaluate on problems containing between 8 and 16 objects, yielding between 28 and 68 state literals.
\item \textbf{Keys and Doors}~(\figref{fig:timeline}). This deterministic domain, inspired by \emph{Lightworld}~\cite{lightworld}, features a robot navigating a gridworld with rooms to reach a goal. There are keys throughout the world, each unlocking some room. We train and evaluate on problems containing between 35 and 93 objects, yielding between 132 and 1169 state literals.
\item \textbf{Triangle Tireworld}~\cite{ipc2008}. Also considered by the two closest prior works, \rex{}~\cite{lang2012exploration} and \ilm{}~\cite{ng2019incremental}, this is the probabilistic IPC Tireworld domain, containing an agent navigating a triangle-shaped network of cities to reach a goal.
With each move, there is some probability that the agent will get a flat tire, and tires can only be changed at certain cities.
We train and evaluate on problems containing between 6 and 15 objects, yielding between 43 and 241 state literals.
\item \textbf{Exploding Blocks}~\cite{ipc2008}. Also considered by \citet{lang2012exploration} and \citet{ng2019incremental}, this is the probabilistic IPC version of Blocks, in which every time the agent interacts with an object, there is a chance that this object is destroyed forever. Therefore, even the optimal policy cannot solve the task 100\% of the time. We train and evaluate on problems containing between 5 and 7 objects, yielding between 31 and 57 state literals.
\item \textbf{PyBullet}. Pictured in \figref{fig:teaser} and inspired by tasks considered by \citet{pasula2007learning} and \citet{lang2012exploration}, this domain can be understood as a continuous and stochastic version of Blocks; a robot simulated in the PyBullet physics engine~\cite{PyBullet} picks and stacks blocks on a table. This domain involves realistic physics and imperfect controllers (e.g., the robot sometimes drops a block when attempting to pick it up); therefore, robustness to stochasticity is important. 
We hand-defined a featurizer that converts from the raw (continuous) state to (discrete) predicate logic, but the state transitions are computed via the simulator. We train and evaluate on problems containing 5 objects, yielding 37 state literals.
\end{tightlist}

\textbf{Exploration methods evaluated:} \begin{tightlist}
    \item \textbf{Oracle}. This method has access to the ground truth model and is intended to provide an approximate upper bound on the performance of an exploration strategy.
    The oracle picks an action for the current state whose most likely predicted effects under the current learned model and ground truth model do not match. If all match, the oracle performs breadth-first search (with horizon 2) in the determinized models, checking for any future mismatches, and falling back to action babbling when none are found. We do not run the oracle for the PyBullet domain because there are no ground truth NDRs for it.
    \item \textbf{Action babbling}. A uniformly random exploration policy over the set of ground actions in the domain.
    \item \textbf{\irale{}}~\cite{rodrigues2011active}. This exploration method uses the current learned model for action selection, but does not perform lookahead with it.
    \item \textbf{\expo{}}~\cite{gil1994learning}. This operator refinement method allows for correcting errors in operators when they are discovered. Since we do not have goals at training time, we run action babbling until an error is discovered.
    \item \textbf{\rex{}}~\cite{lang2012exploration} in E$^3$-exploration mode.
    \item \textbf{\ilm{}}~\cite{ng2019incremental}, which builds on \rex{} by introducing a measure of model reliability.
    \item \textbf{\glibg{} (ours)}. \glib{} in \emph{ground} mode with $k=1$.
    \item \textbf{\glibl{} (ours)}. \glib{} in \emph{lifted} mode with $k=2$. We use a larger value of $k$ in \emph{lifted} mode than \emph{ground} mode because there are typically far fewer lifted goals than ground ones for a given $k$ value, and our preliminary results found that \glibl{} with $k=1$ was never better than \glibl{} with $k=2$.
\end{tightlist}

\textbf{Evaluation.} We evaluate the learned models in terms of prediction error, measured as the percentage of randomly sampled ground-truth transitions that are not predicted by the learned model as most likely, and planning performance, measured on a suite of planning problems. We ensure that all goals in the planning problems are sufficiently large (length 3 or more) conjunctions of literals so that they could not possibly be babbled by the agent during \glib{} exploration.

\subsection{Results and Discussion}
\begin{figure}[t]
  \centering
    \noindent
    \includegraphics[width=\columnwidth]{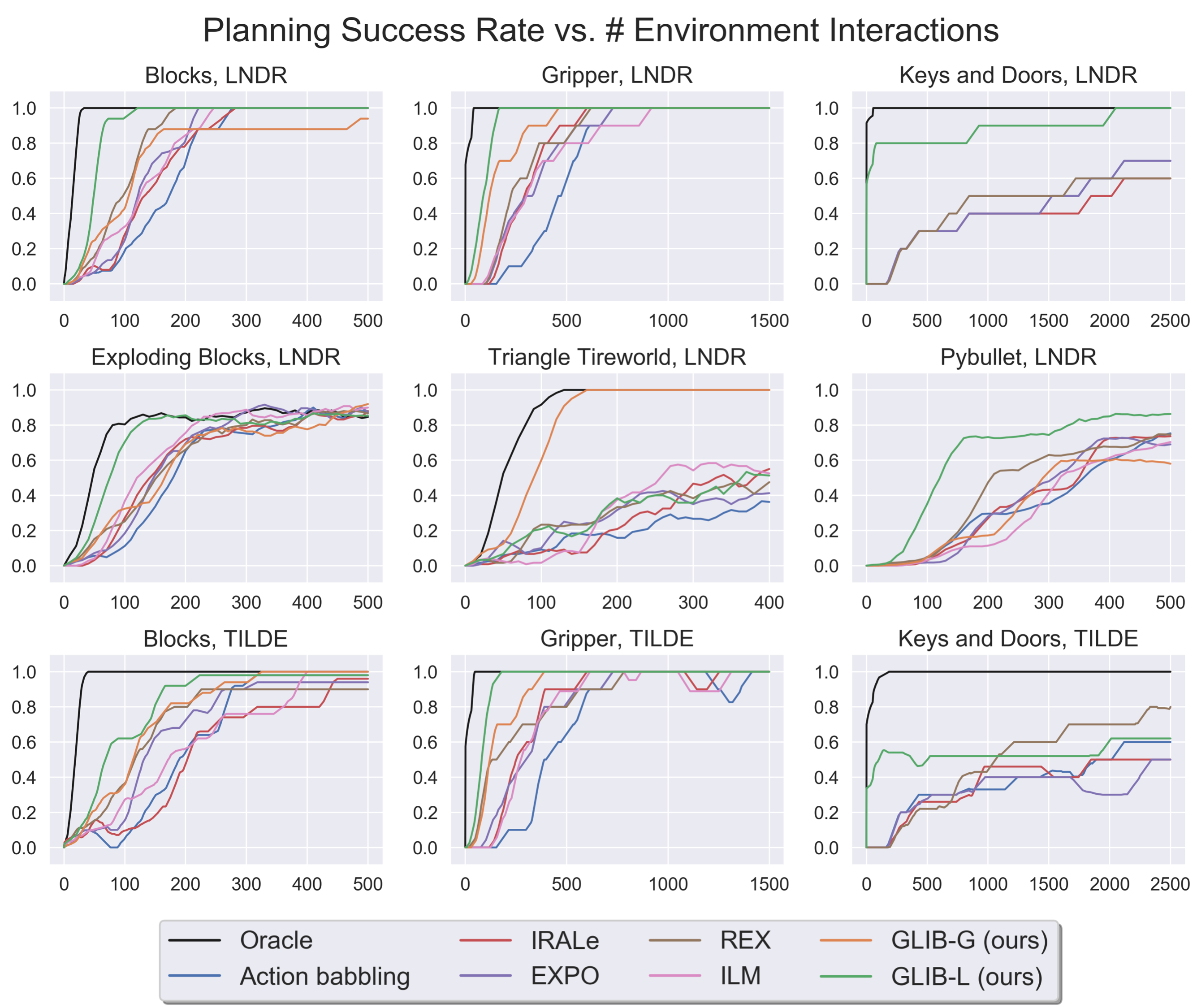}
    \caption{Success rate on planning problems (higher is better) versus number of environment interactions. Top two rows use the \zpk{} model learner, and bottom row uses the \textsc{tilde} model learner (which only works on deterministic domains). All curves show averages over 10 seeds. 
    Standard deviations are omitted for visual clarity. In all domains, \glibl{} performs substantially better than all other methods, except in Triangle Tireworld, where \glibg{} does so. Oracle is not run in PyBullet because a ground truth model is not available. \glibg{} and \ilm{} are not run on Keys and Doors due to the large space of ground literals in this domain.}
  \label{fig:resultssucc}
\end{figure}

\figref{fig:resultssucc} shows all results for planning problem success rates as a function of the number of environment interactions.
\figref{fig:resultspred} shows all results for prediction error rates as a function of the number of environment interactions. 
It is clear, especially from \figref{fig:resultssucc}, that \glib{} performs substantially better than all other approaches, whether in \emph{ground} mode for Triangle Tireworld or in \emph{lifted} mode for all other domains. In some domains, such as Keys and Doors, exploration with \glib{} is up to two orders of magnitude more data-efficient than all the baselines. In the Keys and Doors domain, to open the door to a room, the agent must first move to and pick up the key to unlock that door; in these bottleneck situations, \glib{} is able to shine, as the agent often sets goals that drive itself through and beyond the bottleneck.

\glibg{} sharply outperforms \glibl{} in Triangle Tireworld because there are very few predicates in this domain; just by randomly interacting with the world for a few timesteps, the agent can see nearly all possible conjunctions of two lifted literals, and so \glibl{} with $k=2$ has no remaining goals to babble. On the other hand, \emph{ground} goals continue to be interesting, and so \glibg{} allows the agent to set itself goals such as reaching previously unvisited locations. This result illustrates that the choice of \glibl{} or \glibg{} depends greatly on properties of the domain.

These results suggest that \glib{} is a strong approach for exploration; a natural next question is how long \glib{} takes. In \tabref{tab:timings} of Appendix \ref{app:timings}, we show that the per-iteration speed of \glib{}, especially in \emph{lifted} mode, is competitive with that of the two closest prior works, \rex{} and \ilm{}.
We found filtering out static and mutex goals was necessary for making \glib{}'s speed competitive, but did not affect Figures \ref{fig:resultssucc} and \ref{fig:resultspred}.

\begin{figure}[t]
  \centering
    \noindent
    \includegraphics[width=\columnwidth]{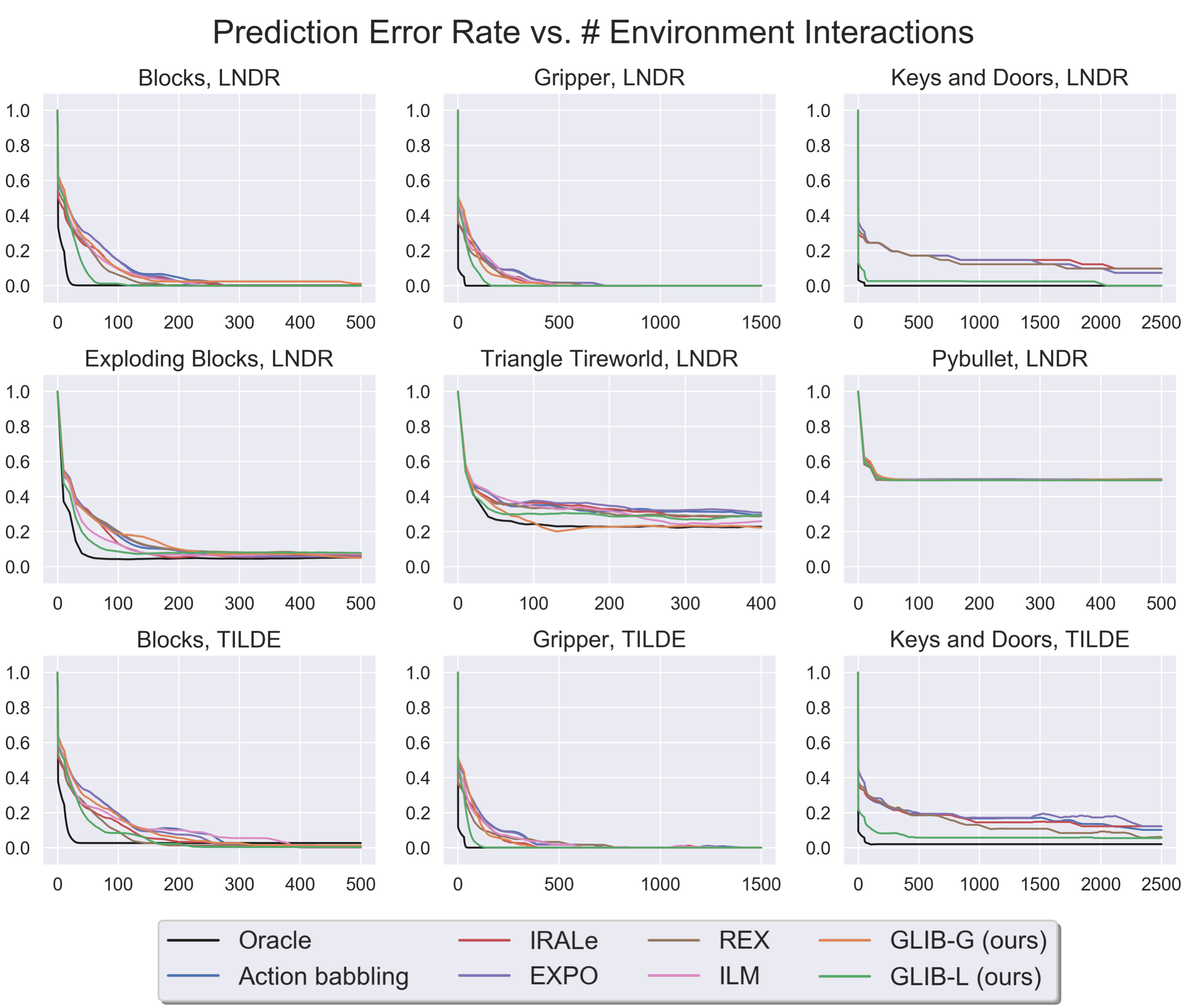}
    \caption{Prediction error rates (lower is better); see \figref{fig:resultssucc} caption for details. These results together with \figref{fig:resultssucc} make clear that modest advantages in prediction error can translate into dramatic gains during planning.}
  \label{fig:resultspred}
\end{figure}

\section{Conclusion}

We have introduced Goal-Literal Babbling (\glib{}) as a simple, efficient exploration method for transition model learning in relational model-based reinforcement learning.
We showed empirically that \glib{} is a very strong exploration strategy, in some cases achieving up to two orders of magnitude better sample efficiency than prior approaches.

There are several useful directions for future work. One is to develop better fallback strategies, for instance, planning to get as close to a babbled goal as possible when the goal cannot be reached. While this would require an additional assumption in the form of a metric over the state space, it may help the agent better exploit the implicit Voronoi bias resulting from bootstrapping exploration with goal-directed search under the current learned model. 
Another line of work could be to combine \glib{} with other exploration methods; for instance, one could combine the insights of \rex{} and \glib{}, planning for long horizons but only within known or ``trusted'' parts of the state space under the current model.

\newpage
\section*{Acknowledgements}
We would like to thank Ferran Alet and Caris Moses for
their valuable comments on an initial draft. We gratefully
acknowledge support from NSF grant 1723381; from AFOSR grant FA9550-17-1-0165; from ONR
grant N00014-18-1-2847; from the Honda Research Institute; from MIT-IBM Watson Lab; and from
SUTD Temasek Laboratories. Rohan and Tom are supported by NSF Graduate Research Fellowships.
Any opinions, findings, and conclusions or recommendations expressed in this material are those of
the authors and do not necessarily reflect the views of our sponsors.
\bibliography{glibbib}

\cleardoublepage
\appendix

\section{Additional Experiment Details}
\label{app:details}

\subsection{Incremental Model Learning}
\citet{ng2019incremental} propose an extension of \zpk{} to the incremental regime, where the transition model is progressively improved as more data is collected.
Their strategy is to penalize NDR sets that deviate far from the previously learned set during the greedy search.
We found that simply initializing the greedy search with the previously learned NDR set was sufficient to attain consistently good performance.
For \textsc{tilde}, we introduced two extensions to the method that allows it to be applied to the model-learning setting: we allow for lifted literal classes (in order to describe lifted effects), and we allow multiple output literals for a single input (in order to capture conjunctive effects). 
Note that \textsc{tilde} is only applicable in deterministic domains.
To improve overall runtime, we retrain the model only when a new transition disagrees with the most likely transition predicted by the current model.

\subsection{Software and Hyperparameters}
For interacting with relational environments, we use the PDDLGym library~\cite{pddlgym}, version 0.0.1. All experiments were conducted on a quad-core AMD64 processor with 4GB RAM, in Ubuntu 18.04.

We use $T=25$ as the episode length for all domains, except for Triangle Tireworld which uses $T=8$ and PyBullet which uses $T=10$. We chose these values for Tireworld and PyBullet because $T=25$ proved to be unnecessarily large in these domains.
All methods use FF-Replan \cite{ffreplan} with single-outcome determinization as the planner; FF-Replan uses Fast-Forward~\cite{ff}. All planning calls have a timeout of 10 seconds. We set $N$, the number of sampling tries in \algref{alg:glib}, to 100. Model learning has a timeout of 3 minutes per iteration, at which point we use the best model discovered so far. We did not perform much tuning on these hyperparameters; the results are already quite strong, but they could be improved even further via a grid search.

\section{Timing Results}
\label{app:timings}
Table \ref{tab:timings} presents timing results for \glib{} and baseline methods, showing that \glib{}'s strong performance does not come at the expense of time.
In the Gripper domain, \ilm{} is quite slow; this is because calculating the count for the current state on each iteration requires looping over the dataset to estimate applicability of each NDR.

\begin{table}[t]
  \centering
  \resizebox{0.8\columnwidth}{!}{
  \tabcolsep=0.1cm{
  \begin{tabular}{c|cccccc}
    \toprule[1.5pt]
     & \textbf{BL} & \textbf{GR} & \textbf{KD} & \textbf{EB} & \textbf{TT} & \textbf{PY}\\
    \midrule[2pt]
    \textbf{Action babbling} & 0.0 & 0.0 & 0.0 & 0.0 & 0.0 & 0.0\\
    \midrule
    \textbf{\irale{}} & 0.0 & 0.0 & 0.3 & 0.0 & 0.0 & 0.0\\
    \midrule
    \textbf{\expo{}} & 0.1 & 0.0 & 0.4 & 0.1 & 0.0 & 0.1\\
    \midrule
    \textbf{\rex{}} & 0.0 & 0.1 & 1.1 & 0.0 & 0.0 & 0.2\\
    \midrule
    \textbf{\ilm{}} & 0.2 & 28.9 & --- & 0.2 & 0.4 & 0.6\\
    \midrule[1.5pt]
    \textbf{\glibg{} (ours)} & 0.2 & 0.1 & --- & 0.4 & 0.4 & 0.8\\
    \midrule
    \textbf{\glibl{} (ours)} & 0.1 & 0.1 & 0.2 & 0.1 & 0.2 & 0.2\\
    \bottomrule[1.5pt]
  \end{tabular}}}
  \caption{Average seconds per iteration taken by each exploration method. Each column is a domain: BL = Blocks, GR = Gripper, KD = Keys and Doors, EB = Exploding Blocks, TT = Triangle Tireworld, PY = PyBullet. Every number is an average over 10 random seeds. The number 0.0 indicates that the time is $<0.05$, not precisely zero. All times are obtained using the \zpk{} model learner. We can see that the speed of \glib{}, especially in \emph{lifted} mode, is competitive with that of all baselines, especially the ones which perform lookahead for exploration (\rex{} and \ilm{}). \glibg{} and \ilm{} are intractable on Keys and Doors because the space of ground literals is prohibitively large in this domain.}
  \label{tab:timings}
\end{table}

\section{Proofs}
\setcounter{lem}{0}
\setcounter{theorem}{0}
\label{app:proofs}
\begin{lem}
Suppose that \textsc{LearnModel} is consistent. Given any state $s_0 \in \S$, goal $G$, and policy $\pi$, consider transitions sampled from the ground truth distribution $P$ by repeatedly starting at $s_0$ and following $\pi$ for $T$ steps.
Let $\hat{P}_t$ be the transition model returned by calling \textsc{LearnModel} on the first $t$ transitions.
Then a.s., either
(1) $G$ is eventually reached; or
(2) the probability that $\pi$ would reach $G$ from $s_0$ under $\hat{P}_t$ converges to 0 as $t \to \infty$.
\end{lem}
\begin{proof}[Proof sketch]
If the probability that $\pi$ reaches $G$ from $s_0$ under the ground truth model is positive, then $G$ will a.s. be reached.
Otherwise, for any sequence of states and actions of length at most $T$ that starts at $s_0$, follows $\pi$, and ends at $G$, there must be some transition $(s, a, s')$ s.t. $P(s' | s, a) = 0$. 
Consider the first such transition $(s, a, s')$ in a particular sequence.
With probability 1, the state-action pair $(s, a)$ will be seen in the sampled transitions infinitely many times.
Since \textsc{LearnModel} is consistent, for any $\epsilon > 0$, there will a.s. be some time $\tau$ such that for all $t > \tau, \hat{P}_t(s' | s, a) < \epsilon$; the probability of the overall sequence reaching $G$ must also then be less than $\epsilon$.
Thus, since all of the sequences starting at $s_0$, following $\pi$, and ending at $G$ a.s. have probabilities converging to 0, and there are finitely many sequences given that $\S, \A$ and $T$ are finite, the total probability of $\pi$ reaching $G$ from $s_0$ under $\hat{P}$ also a.s. converges to 0.
\end{proof}

\begin{theorem}[\glib{} is a.s. recurrent]
Suppose that the environment is recurrent, \textsc{LearnModel} is consistent, and \textsc{Plan} is $\epsilon$-sound. Then for any integer $k > 0$, $\textsc{MC}(\glib{}(k), \I, \S, \A)$ is a.s. recurrent over $\Omega$.
\end{theorem}
\begin{proof}
We begin by showing that a.s., any goal can only be babbled by \glib{} finitely many times. To see this, suppose toward a contradiction that there is a goal $G$ that, with measure greater than 0, is babbled infinitely many times. 
Because the state space is finite, there must exist some starting state $s \in \S$ from which $G$ is babbled infinitely many times. 
Because both the state space and the action space are finite, the space of policies is also finite; therefore, there must exist some policy $\pi$ for achieving $G$ from $s$ that is returned by \textsc{Plan} infinitely many times, but never successfully reaches $G$ (because \glib{} only babbles novel goals). 
Lemma \ref{lem:lemma1} states that a.s., $G$ is eventually reached or eventually considered unreachable, within probability $\epsilon$, under the learned model.
In the latter case, since \textsc{Plan} is $\epsilon$-sound, and since goals are only babbled if some policy is found for achieving that goal, $G$ would be babbled only finitely many times.
Thus we have a contradiction; a.s., each goal is babbled by \glib{} only finitely many times.

Since any goal can a.s. only be babbled finitely many times, and there are finitely many goals, there exists a timestep after which \glib{} a.s. has no more goals to babble. After this, \glib{} will constantly fall back to taking random actions, so its behavior will become equivalent to \textsc{Random}. The a.s. recurrence of \glib{} follows from Definition \ref{defn:recurrent}.
\end{proof}

\end{document}